\documentclass{article}
\usepackage{ijcai25}

\usepackage{times}
\usepackage{soul}
\usepackage{url}
\usepackage[hidelinks]{hyperref}
\usepackage[utf8]{inputenc}
\usepackage[small]{caption}
\usepackage{graphicx}
\usepackage{amsmath}
\usepackage{amsthm}
\usepackage{booktabs}
\usepackage{algorithm}
\usepackage{algorithmic}
\usepackage[switch]{lineno}

\usepackage{amsmath}
\usepackage{amssymb}
\usepackage{mathtools}
\usepackage{amsthm}
\usepackage{paralist}

\newtheorem{theorem}{Theorem}
\newtheorem{lemma}[theorem]{Lemma}

\title{Incentivizing Safer Actions in Policy Optimization for \\ Constrained Reinforcement Learning}

\author{
Somnath Hazra$^1$
\and
Pallab Dasgupta$^2$\And
Soumyajit Dey$^1$\\
\affiliations
$^1$Indian Institute of Technology Kharagpur, India\\
$^2$Synopsys, USA\\
\emails
somnathhazra@kgpian.iitkgp.ac.in,
pallabd@synopsys.com,
soumya@cse.iitkgp.ac.in
}

\begin{document}

\maketitle

\insert\footins{Code and Supplementary Material available here: \href{https://github.com/somnathhazra/IP3O}{https://github.com/somnathhazra/IP3O}.}

\begin{abstract}
    Constrained Reinforcement Learning (RL) aims to maximize the return while adhering to predefined constraint limits, which represent domain-specific safety requirements. In continuous control settings, where learning agents govern system actions, balancing the trade-off between reward maximization and constraint satisfaction remains a significant challenge. Policy optimization methods often exhibit instability near constraint boundaries, resulting in suboptimal training performance. To address this issue, we introduce a novel approach that integrates an adaptive incentive mechanism in addition to the reward structure to stay within the constraint bound before approaching the constraint boundary. Building on this insight, we propose Incrementally Penalized Proximal Policy Optimization (IP3O), a practical algorithm that enforces a progressively increasing penalty to stabilize training dynamics. Through empirical evaluation on benchmark environments, we demonstrate the efficacy of IP3O compared to the performance of state-of-the-art Safe RL algorithms. Furthermore, we provide theoretical guarantees by deriving a bound on the worst-case error of the optimality achieved by our algorithm.
\end{abstract}

\section{Introduction}
\label{sec1}

Constraint satisfaction represents a critical challenge in Reinforcement Learning (RL), particularly due to its significant implications in real-world applications such as robotics \cite{levine2016end,ono2015chance}, autonomous driving \cite{fisac2018general,kiran2021deep,fernandez2023trustworthy}, healthcare \cite{yu2021reinforcement}, finance \cite{mcnamara2016law}, etc. These tasks involve learning a policy to address sequential decision-making problems to maximize return under predefined safety constraints and accounting for uncertainties in the environment. Such requirements are commonly formulated using the Constrained Markov Decision Process (CMDP) framework \cite{altman2021constrained}, which extends the conventional Markov Decision Process (MDP) to accommodate safety-critical considerations. Consequently, the applicability of traditional RL methods becomes limited in these settings, necessitating the development of new approaches to address constraint satisfaction effectively.

Traditional RL algorithms focus solely on maximizing the return \cite{schulman2017proximal}. To satisfy the cost budget, the safe RL methods typically incorporate an additional term into the loss function. However, since the cost component is not inherently optimized, the imposed penalty significantly dictates the resulting policy's behavior; it should be rewarding enough as far as the cost constraint limits are respected. For instance, Constrained Policy Optimization (CPO) \cite{achiam2017constrained} uses a second-order expansion term to approximate the cost but suffers from high computational overhead and suboptimal performance due to approximation errors. Primal-dual methods transform the constrained problem into an unconstrained dual via Lagrangian multipliers \cite{tessler2018reward,ding2020natural,yu2019convergent,dai2023augmented}, yet these approaches are prone to constraint violations and oscillatory behavior in practical scenarios \cite{stooke2020responsive}. Projection-based methods apply a secondary optimization to project unsafe actions into a feasible region \cite{zhang2020first,yang2020projection,yang2022constrained}; however, these methods are computationally expensive and can underperform in certain environments \cite{dai2023augmented}.

In contrast, penalty function approaches directly encode the constraints in the primal space itself. By introducing a penalty term into the loss function based on the constraint cost, these methods establish a penalty barrier to regulate policy behavior \cite{liu2020ipo,zhang2022penalized,zhang2023evaluating}. The choice of barrier function is critical, as it determines the trade-off between performance and constraint satisfaction \cite{zhang2023evaluating}. Despite their effectiveness, most existing approaches impose penalties only after a constraint violation occurs, offering no proactive incentive for the policy to remain within predefined safety limits. This strategy often compromises the stability of training and constraint satisfaction. To address this limitation, we propose a novel approach that transforms the cost function into a positive incentive within the safe region and gradually converts it into a penalty as the safety boundary is breached. The design ensures a more dynamic adherence to constraints, leading to more stable and efficient optimization.

\begin{figure}[!ht]
    \centering
    \includegraphics[width=0.75\linewidth]{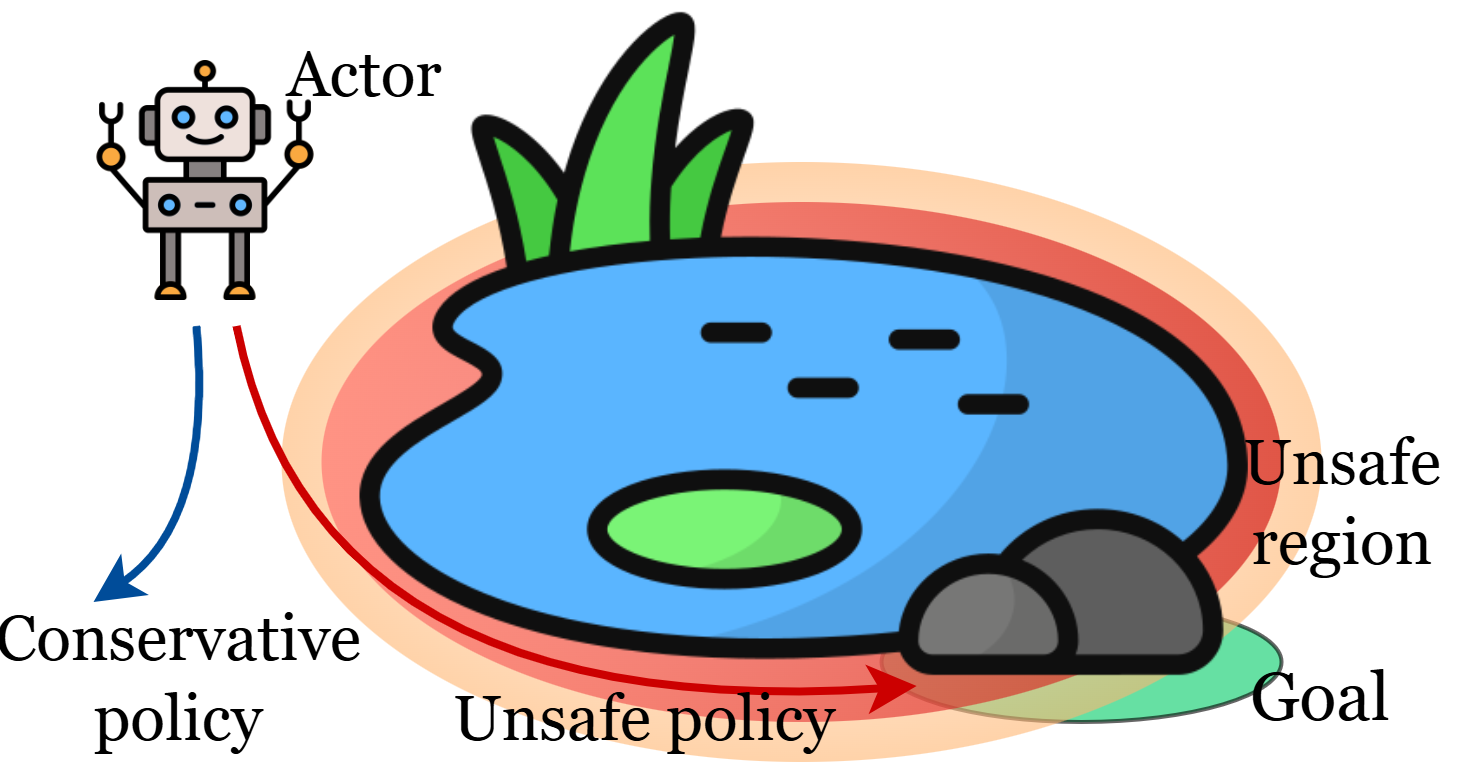}
    \caption{Illustration of a risky trajectory and a conservation trajectory under unknown environmental dynamics.}
    \label{fig1}
\end{figure}

To illustrate our hypothesis, consider a robot tasked with reaching the rock, i.e. the goal, located on the opposite side of a pond while avoiding falling into the water (Figure \ref{fig1}). The shortest path, marked in red, is the most efficient but also the riskiest due to environmental uncertainties. While on-policy algorithms in RL can mitigate such risks \cite{sutton2018reinforcement}, given the cost function, the reward function alone may fail to discourage unsafe actions effectively in a CMDP. Introducing an incentive derived from the cost function encourages safer behavior. However, excessive incentives for avoiding risk can result in overly conservative policies that fail to achieve the goal, as indicated by the blue path in the figure. To address this trade-off, we propose Incrementally Penalized Proximal Policy Optimization (IP3O), which incorporates an adaptive penalty mechanism to balance safety and performance under general model-free settings.

The primary contributions are summarized as follows:
\begin{itemize}
    \item We design a novel penalty function to adaptively incentivize safe actions based on the cost function, ensuring that safety constraints are satisfied without compromising return optimization.
    
    \item We propose the Incrementally Penalized Proximal Policy Optimization (IP3O) algorithm, which integrates the penalty barrier function to achieve efficient learning of safe policies. Additionally, we provide theoretical guarantees by deriving a worst-case performance bound for the algorithm.
    
    \item We conduct extensive empirical evaluations of IP3O against multiple state-of-the-art approaches in safe RL across benchmark environments, demonstrating its superior performance in balancing safety and return.
\end{itemize}
These contributions collectively address key challenges in constrained RL by bridging the gap between safety and performance in real-world scenarios. We next present the relevant background in the Preliminaries section.

\section{Preliminaries}
\label{sec2}

RL problems are modeled using a Markov Decision Process (MDP), defined by the tuple $\langle \mathcal{S, A, R}, P, \rho, \gamma \rangle$, where $\mathcal{S}$ denotes the state space, $\mathcal{A}$ denotes the action space, $\mathcal{R}: \mathcal{S} \times \mathcal{A} \to \mathbb{R}$ is the reward function, $P: \mathcal{S} \times \mathcal{A} \times \mathcal{S} \to [0, 1]$ is transition probability function, $\rho: \mathcal{S} \to [0, 1]$ is the initial state distribution, and $\gamma \in (0, 1)$ is the discount factor for calculating the return. In model-free settings, $P$ is unknown. The objective is to find a policy $\pi: \mathcal{S} \to \Delta(\mathcal{A})$ that maximizes the expected discounted return: $\mathcal{J}_{\mathcal{R}}(\pi) = \mathbb{E}_{\tau \sim \pi} [ \sum_{t=0}^{\infty} \gamma^t \mathcal{R}(s_t, a_t) ]$, where $\tau$ is the trajectory. The utility of a state $s$ is measured using the value function: $V_{\mathcal{R}}^{\pi}(s) = \mathbb{E}_{\tau \sim \pi} [ \sum_{t=0}^{\infty} \gamma^t \mathcal{R}(s_t, a_t) | s_0 = s]$; and that after taking action $a$ at state $s$ is given by the state-action value function: $Q_{\mathcal{R}}^{\pi}(s, a) = \mathbb{E}_{\tau \sim \pi} [ \sum_{t=0}^{\infty} \gamma^t \mathcal{R}(s_t, a_t) | s_0 = s, a_0 = a]$. The advantage of taking action $a$ at $s$ is given by the advantage function: $A_{\mathcal{R}}^{\pi}(s, a) = Q_{\mathcal{R}}^{\pi}(s, a) - V_{\mathcal{R}}^{\pi}(s)$.

The Constrained MDP (CMDP) introduces the cost function, $\mathcal{C}: \mathcal{S} \times \mathcal{A} \rightarrow \mathbb{R}$, and is represented using the tuple $\langle \mathcal{S, A, R, C}, P, \rho, \gamma \rangle$. The allowable policy set $\Pi_{\mathcal{C}} \subset \Pi$ is restricted such that the expected cumulative discounted cost satisfies $\mathcal{J}_{\mathcal{C}}(\pi) \leq d$, where $d$ is the constraint threshold. The value functions, $V_{\mathcal{C}}^{\pi}, Q_{\mathcal{C}}^{\pi}, A_{\mathcal{C}}^{\pi}$ are similarly defined as in MDPs using $\mathcal{C}$. Thus, constrained RL aims to find a policy $\pi \in \Pi_{\mathcal{C}}$, where:
\begin{equation}
    \label{eqn1}
    \arg \max_{\pi} \mathcal{J}_{\mathcal{R}}(\pi) \quad \text{s.t. } \mathcal{J}_{\mathcal{C}}(\pi) \leq d
\end{equation}

A more general consideration is where $\mathcal{C}$ represents a set of $m$ constraints. In those cases, the objective is to find a policy where $\arg \max_{\pi} \mathcal{J}_{\mathcal{R}}(\pi) \quad \text{s.t. } \mathcal{J}_{\mathcal{C}_i}(\pi) \leq d_i$. In the rest of the text we assume the multi-constraint formulation to describe our approach.

\section{Related Works}
\label{sec3}

Constrained RL in model-free settings remains a challenging problem due to the difficulty in converging to near-optimal solutions under high-dimensional approximations and environmental uncertainties. Recent research has introduced various approaches to address these challenges, focusing on effective constraint handling during policy optimization.

The constrained RL objective is often built as a dual problem using Lagrangian relaxation \cite{chow2018risk}, converting episodic constraints into an unconstrained dual objective. RCPO \cite{tessler2018reward} penalizes policy updates based on constraint violations with a Lagrangian multiplier. However, dual optimization adds oscillatory behavior, mitigated in CPPOPID \cite{stooke2020responsive} with a PID controller, or using an Augmented Lagrangian in APPO \cite{dai2023augmented}. While sometimes effective, these methods are very sensitive to Lagrangian parameters, limiting their applicability.

On-policy RL algorithms, that update the existing policy within a trust region, have been extensively used in constrained RL to safely update the existing policy. Algorithms such as TRPO \cite{schulman2015trust} enforce trust-region updates through divergence constraints, while PPO \cite{schulman2017proximal} simplifies this using gradient clipping. Building on TRPO, CPO \cite{achiam2017constrained} integrates second-order approximation of constraints; but suffers from the computational overhead of Fisher Information Matrix inversion. Methods such as PCPO \cite{yang2020projection} and CUP \cite{yang2022constrained} adopts two-phase optimization: policies are updated and then projected back into the feasible space. FOCOPS \cite{zhang2020first} introduces a non-parametric stage for feasible policy generation, followed by parametric projection. However, these projection-based approaches often fail to guarantee optimality, highlighting a critical limitation.

Penalty barrier methods impose a penalty based on the constraint violation, restricting optimization within predefined barriers. IPO \cite{liu2020ipo} uses a logarithmic penalty function; whereas P3O \cite{zhang2022penalized} employs a ReLU-based penalty. Yet, these methods activate penalties only after constraints are breached, leading to an non-smooth transition in learning dynamics. The reward function may not incentivize a safe action. As described in the next section, our work builds on this line of research by introducing a smooth transition from incentives to penalties, facilitating stable policy learning while ensuring efficient adherence to constraints.

\section{Methodology}
\label{sec4}

Policy optimization in reinforcement learning improves a policy iteratively by balancing exploration and exploitation of the current policy's knowledge. The performance difference between two policies is given by the following equation \cite{kakade2002approximately}:

\begin{equation}
    \label{eqn2}
    \mathcal{J}_{\mathcal{R}}(\pi') - \mathcal{J}_{\mathcal{R}}(\pi) = \frac{1}{1 - \gamma} \mathbb{E}_{\substack{s \sim d^{\pi'}\\a \sim \pi'}} \left[ A_{\mathcal{R}}^{\pi}(s, a) \right]
\end{equation}
where $d^{\pi}(s) = (1 - \gamma)\sum_{t=0}^{\infty} \gamma^t P(s_t=s | \pi)$ is the discounted state distribution under policy $\pi$. Extending this to cost functions $\mathcal{C}$, the constrained optimization objective (\ref{eqn1}) can be reformulated as follows:
\begin{equation}
    \label{eqn3}
    \begin{gathered}
    \pi_{k+1} = \arg \max_{\pi} \mathbb{E}_{\substack{s \sim d^{\pi}\\a \sim \pi}} \left[ A_{\mathcal{R}}^{\pi_k}(s, a) \right] \\
    \text{s.t.} \quad \mathcal{J}_{\mathcal{C}_i}(\pi_k) + \frac{1}{1 - \gamma} \mathbb{E}_{\substack{s \sim d^{\pi}\\a \sim \pi}} \left[ A_{\mathcal{C}_i}^{\pi_k}(s, a) \right] \leq d_i
    \end{gathered}
\end{equation}
where $\pi_k$ is the current policy. Given the continuous state space, we use a parametric policy as $\pi(\theta_k) = \pi_k$. To approximate $A_{\mathcal{R}}^{\pi_k}$ and $A_{\mathcal{C}_i}^{\pi_k}$ using samples from the previous policy, $\pi_k$, the importance sampling ratio is used: $r(\theta) = \frac{\pi(\theta)}{\pi(\theta_k)}$. This modifies the optimization problem as follows.
\begin{equation}
    \label{eqn4}
    \begin{gathered}
        \pi_{k+1} = \arg \max_{\pi} \mathbb{E}_{\substack{s \sim d^{\pi_k}\\a \sim \pi_k}} \left[ r(\theta) A_{\mathcal{R}}^{\pi_k}(s, a) \right] \\
    \text{s.t.} \quad \frac{1}{1 - \gamma} \mathbb{E}_{\substack{s \sim d^{\pi_k}\\a \sim \pi_k}} \left[ r(\theta) A_{\mathcal{C}_i}^{\pi_k}(s, a) \right] + \mathcal{J}_{\mathcal{C}_i}(\pi_k) \leq d_i
    \end{gathered}
\end{equation}
To stabilize the policy update within a trust region, the ratio $r(\theta)$ is clipped within $(1 - \epsilon, 1 + \epsilon)$, where $\epsilon$ is the clipping threshold. In penalty function approaches, constraint violations are penalized through barrier functions added to the reward maximization objective. Existing works \cite{liu2020ipo,zhang2022penalized,gao2024exterior} introduce penalties only after constraints are violated, leading to abrupt learning transitions and suboptimal policies due to risk-averse updates.

Our approach addresses these limitations by introducing an \emph{Incrementally Penalized Proximal Policy Optimization (IP3O)} algorithm. The key features of our approach are:
\begin{itemize}
    \item Incentivizing safe behavior: When constraints are satisfied, an incentive is encoded in the penalty barrier function to guide the policy toward constraint satisfaction without abrupt penalties. 
    \item Gradual transition to penalties: For unsafe trajectories, a smooth transition to penalization is incorporated, mitigating sudden changes in the learning dynamics.
\end{itemize}
This formulation ensures stable policy updates, enabling a smoother interplay between reward optimization and constraint satisfaction. In this section, we describe the IP3O algorithm in detail, including its implementation and theoretical guarantees. 

\subsection{Penalty Function Design}

Penalty functions, such as Leaky ReLU, allow negative activation for values below zero, incentivizing constraint satisfaction. However, this has a major drawback; their inability to reduce the slope in the negative region leads to excessive incentives for remaining deeply within the constraint region, which may overshadow rewards and result in overly conservative policies that fail to reach the goal (illustrated by the blue path in Figure \ref{fig1}).

To address this issue, we propose using the Exponential Linear Unit (ELU) function \cite{clevert2015fast} as the penalty function. ELU meets two key requirements:
\begin{inparaitem}
    \item smoothly transitions from incentivizing constraint satisfaction to penalizing violations, ensuring stable policy updates.
    \item mitigates the risk of overly conservative policies by nullifying excessive incentives within the constraint region.
\end{inparaitem}
It is defined using the following function.
\begin{equation}
    \label{eqn5}
    \text{ELU}(x, \alpha) =
    \begin{cases}
        x & \text{if } x \geq 0 \\
        \alpha(\exp(x) - 1) & \text{otherwise}
    \end{cases}
\end{equation}
Here, $\alpha$ determines the stagnation point ($-\alpha$); i.e., if $\alpha = 1$ the function gradually reduces to $-1$ for negative values of $x$ and then stabilizes. So higher values promote safety by favoring constraint satisfaction. Despite its advantages, the ELU function introduces gradient discontinuity at $x=0$ when $\alpha \neq 1$ \cite{barron2017continuously}. To overcome this limitation, we employ the Continuously Differentiable ELU (CELU) function that proposes a modification for negative values of $x$ as: $\text{CELU}(x, \alpha) = \alpha(\exp(x/\alpha) - 1), \text{if } x < 0$. CELU maintains the properties of ELU while ensuring gradient continuity, making it more suitable for policy optimization. A comparison among the activation functions and their gradients are shown in Figure \ref{fig3}.

\begin{figure}[!ht]
    \centering
    \includegraphics[width=0.98\linewidth]{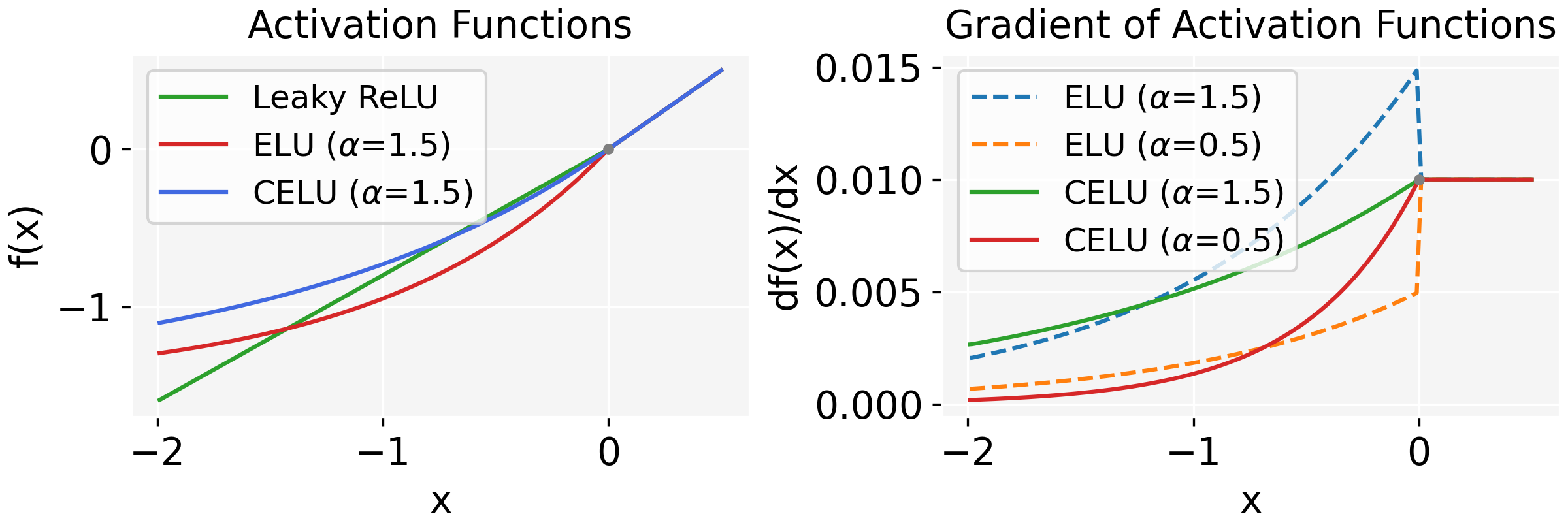}
    \caption{Comaprison of activation functions and their gradients.}
    \label{fig3}
\end{figure}

\subsection{Practical Implementation}

\begin{figure}[!ht]
    \centering
    \includegraphics[width=0.85\linewidth]{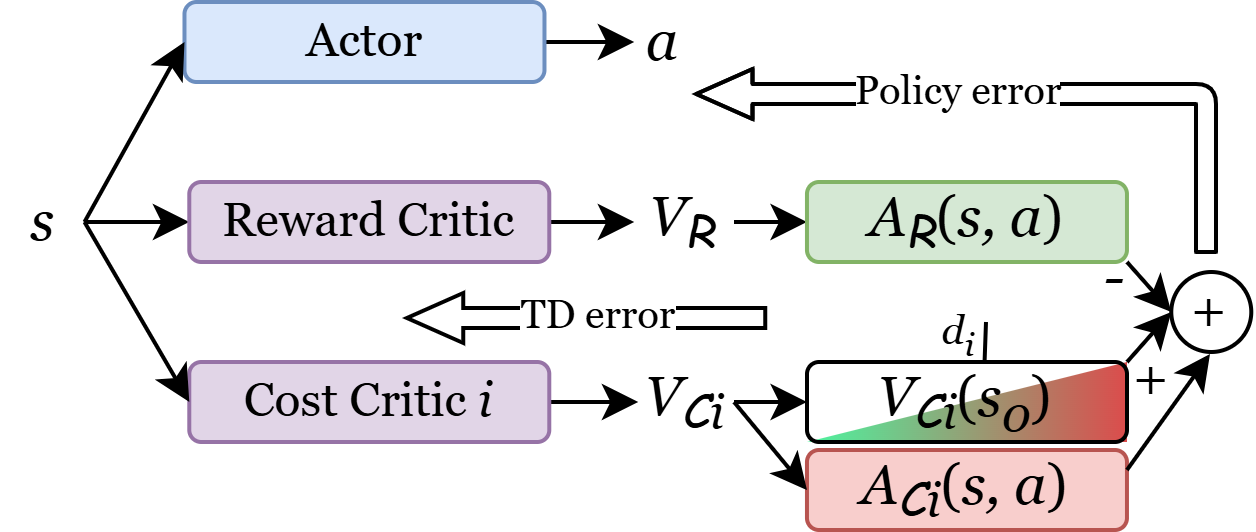}
    \caption{The overall algorithm describing the loss function for training the policy network, with the sources of loss.}
    \label{fig2}
\end{figure}

We integrate the proposed CELU-based penalty function into the cost critic, combining it with the reward critic to guide policy updates. These critics provide gradient signals for optimizing the policy network, as shown in Figure \ref{fig2}. The loss function for the reward is derived from Equation \ref{eqn4} as: $ \mathcal{L}_{\mathcal{R}}(\pi_k) = \mathbb{E}_{\substack{s \sim d^{\pi_k}\\a \sim \pi_k}} [- r'(\theta) A_{\mathcal{R}}^{\pi_k}(s, a)]$, where $r'(\theta) = \min (r(\theta), \text{clip}(r(\theta), 1 - \epsilon, 1 + \epsilon))$. Similarly the loss due to the cost constraint is derived as: $ \mathcal{L}_{\mathcal{C}_i}(\pi_k) = \frac{1}{1 - \gamma} \mathbb{E}_{\substack{s \sim d^{\pi_k}\\a \sim \pi_k}} [ r''(\theta) A_{\mathcal{C}_i}^{\pi_k}(s, a)] + (\mathcal{J}_{\mathcal{C}_i}(\pi_k) - d_i)$, where $r''(\theta) = \max (r(\theta), \text{clip}(r(\theta), 1 - \epsilon, 1 + \epsilon))$. The combined loss function is stated as follows.
\begin{equation}
    \label{eqn6}
    \mathcal{L}(\pi_k) = \mathcal{L}_{\mathcal{R}}(\pi_k) + \eta \sum_{i=1}^m \text{CELU}( \mathcal{L}_{\mathcal{C}_i}(\pi_k))
\end{equation}
where $\eta > 0$ is the penalty factor. This formulation ensures a smooth transition from incentivizing constraint satisfaction to penalizing unsafe behaviors. Since the CELU function uses an exponential term, the gradient of the loss arising from CELU never reaches 0. To ensure no policy updates owing to residual gradients we can use a small positive number $h < \alpha$, such that $\max ( \text{CELU}( \mathcal{L}_{\mathcal{C}_i}(\pi_k)), -\alpha(1 - h) )$ instead of directly using the CELU function in Equation \ref{eqn6}. However in practical applications, we did not require this. To ensure equivalence between the solutions of the proposed loss function (\ref{eqn6}) and the constrained optimization objective (\ref{eqn4}), we establish the following theorem. This result guarantees that the proposed approach preserves the optimality of the original problem while promoting stable and robust policy updates. Hyper-parameter settings, including $\eta$ and $\alpha$ are detailed in the Supplementary Material.

\begin{theorem}
    Given a sequence of policies $\{ \pi_k \}$ obtained by minimizing $\mathcal{L}(\pi_k)$ and considering Slater's condition for strong duality, let $\lambda^*$ be the Lagrange multiplier for the optimal solution of (\ref{eqn4}). If $\eta$ satisfies $\eta \geq || \lambda^* ||_{\infty}$, the limit $\pi^*$ of $\{ \pi_k \}$ is also a solution to (\ref{eqn4}).
\end{theorem}
\begin{proof}
    Provided in the Supplementary Material.
\end{proof}

The value evaluation approximation using the current policy $\pi_k$, instead of the currently evaluating policy, $\pi$, as done in the original problem; introduces a bias \cite{jiang2016doubly}. Moreover the approximation of the loss due to the cost function used may induce a penalty on the optimal reward resulting from the optimization of the original problem in Equation \ref{eqn4}. Using the following theorem we establish the worst-case error bound for these approximations.

\begin{theorem}
    If the loss function $\mathcal{L}(\pi_k)$ is optimized instead of the original problem (\ref{eqn4}), the upper bound on the error is given by the following.
    \begin{equation}
        \frac{\sqrt{2 \delta} \gamma \varepsilon_{\mathcal{R}}^{\pi}}{1 - \gamma} + \eta \sum_{i = 1}^m \left[ \frac{\sqrt{2 \delta} \gamma \varepsilon_{\mathcal{C}_i}^{\pi}}{1 - \gamma} + | \alpha \log(h) | \right]
    \end{equation}
    where $\varepsilon_{\mathcal{R}}^{\pi} = \max_s| \mathbb{E}_{a \sim \pi}[A_{\mathcal{R}}^{\pi_k}(s, a)] |$, $\varepsilon_{\mathcal{C}_i}^{\pi} = \max_s| \mathbb{E}_{a \sim \pi}[A_{\mathcal{C}_i}^{\pi_k}(s, a)] |$, and $\delta = \mathbb{E}_{s \sim d^{\pi_k}}[ D_{KL}(\pi || \pi_k)[s] ]$. 
\end{theorem}
\begin{proof}
    Provided in the Supplementary Material.
\end{proof}

In Algorithm \ref{algo1} below we outline the pseudocode for policy updation using our penalty function.


\begin{algorithm}[!htbp]
    \caption{Policy optimization using IP3O}
    \label{algo1}
    \textbf{Input:} Initial policy $\pi_0$, initial value function $V_{\mathcal{R}}^{\pi_0}$, initial cost value function/s $V_{\mathcal{C}_i}^{\pi_0}$ \;
    \begin{algorithmic}[1]
        \FOR{$k=0, ..., K-1$}
            \STATE Sample training batch $\mathcal{D}_k = \{\tau_1, ..., \tau_N\}$ consisting of $N$ trajectories using $\pi_k$\;
            \STATE Compute $V_{\mathcal{R}}^{\pi_k}, V_{\mathcal{C}_i}^{\pi_k}$ for trajectories in $\mathcal{D}_k$\;
            \STATE \# Advantage calculation \;
            \STATE Compute $A_{\mathcal{R}}^{\pi_k}(s, a) = Q_{\mathcal{R}}^{\pi_k}(s, a) - V_{\mathcal{R}}^{\pi_k}(s)$ \;
            \STATE Compute $A_{\mathcal{C}_i}^{\pi_k}(s, a) = Q_{\mathcal{C}_i}^{\pi_k}(s, a) - V_{\mathcal{C}_i}^{\pi_k}(s)$ \;
            \STATE Update: $V_{\mathcal{R}}^{\pi_k}(s), V_{\mathcal{C}_i}^{\pi_k}(s) \rightarrow V_{\mathcal{R}}^{\pi_{k+1}}(s), V_{\mathcal{C}_i}^{\pi_{k+1}}(s)$ \;
            \STATE \# Policy update \;
            \FOR{$t=0, ..., T-1$}
                \STATE Compute $\mathcal{L}_{\mathcal{R}}(\pi_k) = \mathbb{E}_{\substack{s \sim d^{\pi_k}\\a \sim \pi_k}} [- r'(\theta) A_{\mathcal{R}}^{\pi_k}(s, a)]$ \;
                \STATE Compute $\mathcal{L}_{\mathcal{C}_i}(\pi_k) = \frac{1}{1 - \gamma} \mathbb{E}_{\substack{s \sim d^{\pi_k}\\a \sim \pi_k}} [ r''(\theta) A_{\mathcal{C}_i}^{\pi_k}(s, a)] + (\mathcal{J}_{\mathcal{C}_i}(\pi_k) - d_i)$ \;
                \STATE Compute $\mathcal{L}(\pi_k)$ using Equation \ref{eqn6} \;
                \STATE $\pi = \pi_k + \omega \cdot \nabla \mathcal{L}(\pi_k)$
                \STATE \# Trust region criterion (Gradient clipping) \;
                \IF{$\mathbb{E}_{s \sim d^{\pi_k}}[D_{KL}(\pi || \pi_k)[s]] \notin [\delta^{-}, \delta^{+}]$}
                    \STATE break \;
                \ENDIF
            \ENDFOR
            \STATE Update $\pi_k \rightarrow \pi_{k+1}$ \;
        \ENDFOR
    \end{algorithmic}
    \textbf{Return:} Trained policy $\pi_K$ \;
\end{algorithm}

The output of the algorithm is the final policy. Here $\omega$ is the learning rate. For simplicity we have shown trust region updates using the KL divergence criterion, but in practice we use the PPO updates through gradient clipping \cite{schulman2017proximal}. In the next section we present the empirical details to demonstrate the efficacy of our approach on some benchmark environments in safe RL.

\begin{figure*}
    \centering
    \includegraphics[width=.98\linewidth]{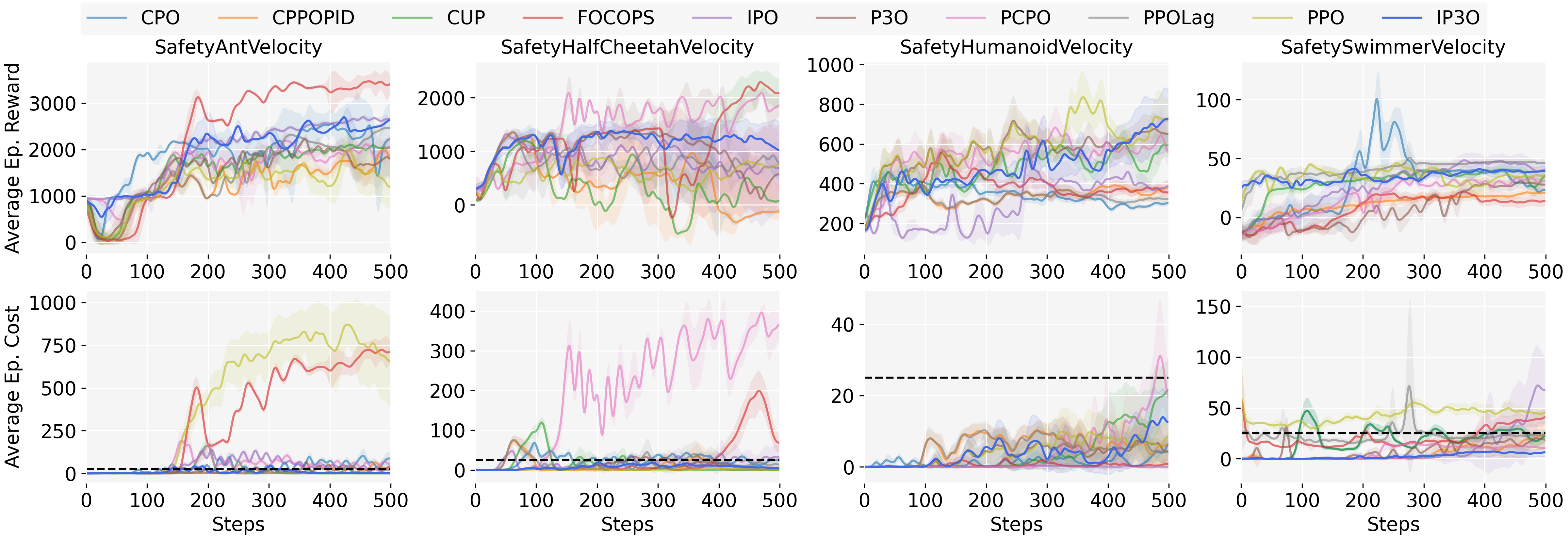}
    \caption{Comparison with the baselines using the MuJoCo safety velocity scenarios. Our method is marked as IP3O (blue). The dashed line indicates the cost constraint limit.}
    \label{fig_res1}
\end{figure*}

\section{Experiments}
\label{sec5}

In this section, we present the evaluation results of our proposed algorithm on benchmark safe RL environments and compare its performance with state-of-the-art approaches. For benchmarking, we consider the following methods.
\begin{itemize}
    \item First order methods such as CUP \cite{yang2022constrained} and FOCOPS \cite{zhang2020first}.
    \item Lagrangian methods such as CPPOPID \cite{stooke2020responsive}, and PPO \cite{schulman2017proximal} with a Lagrangian multiplier.
    \item Second order methods such as CPO \cite{achiam2017constrained}, and PCPO \cite{yang2020projection}.
    \item Penalty function methods such as IPO \cite{liu2020ipo}, and P3O \cite{zhang2022penalized}.
\end{itemize}
All the above algorithms use an on-policy buffer for learning. Additionally, we include a comparison with the vanilla PPO algorithm \cite{schulman2017proximal} as a baseline. All baseline implementations are adapted from the \href{https://github.com/PKU-Alignment/omnisafe}{OmniSafe} repository \cite{ji2024omnisafe}. The evaluations were conducted across three widely-used environments: MuJoCo Safety Velocity \cite{ji2023safety}, Safety Gymnasium \cite{ray2019benchmarking}, and Bullet Safety Gymnasium \cite{gronauer2022bullet}. These environments provide diverse challenges that test the agent’s ability to maximize cumulative rewards while adhering to predefined safety constraints, as formulated in Equation \ref{eqn1}. We also evaluate our approach on multi-agent environments, using the MetaDrive simulator \cite{li2022metadrive}, which are detailed later. Our approach allows tuning of the safety level via the hyper-parameter $\alpha$, which is analyzed further in the Ablation Studies section.

\begin{figure*}[!ht]
    \centering
    \includegraphics[width=.98\linewidth]{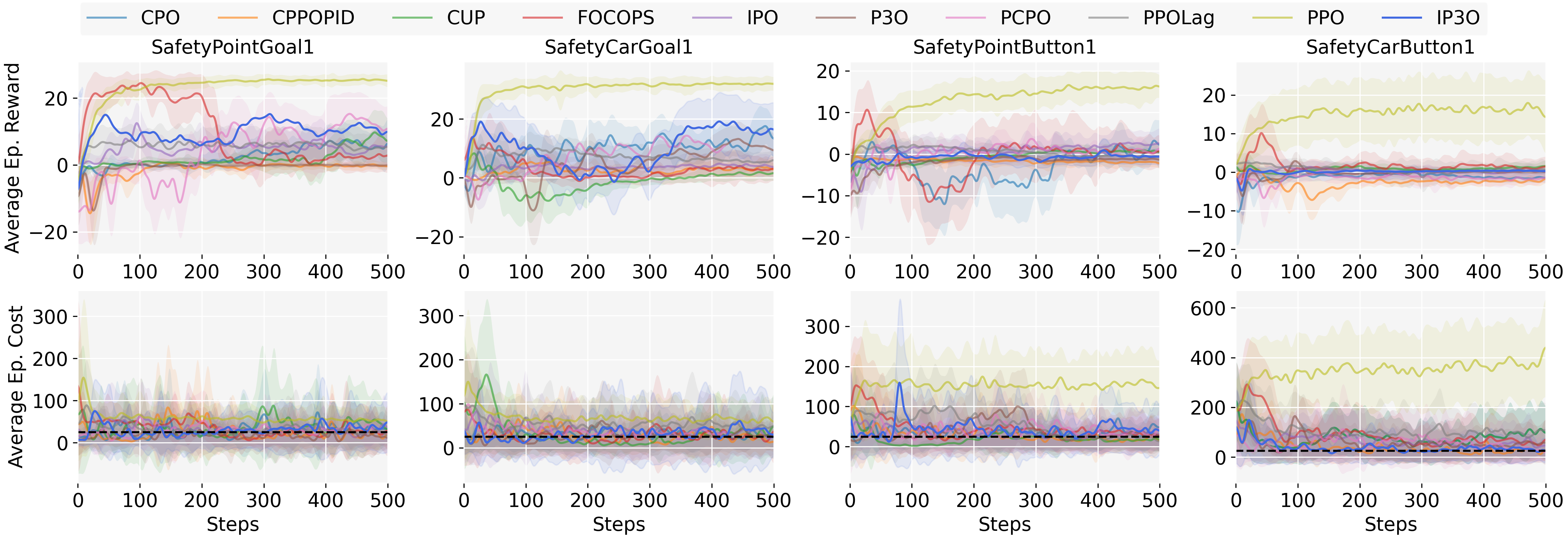}
    \caption{Comparison with the baselines using the Safety Gymnasium scenarios. Our method is marked as IP3O (blue). The dashed line indicates the cost constraint limit.}
    \label{fig_res2}
\end{figure*}

\begin{figure*}
    \centering
    \includegraphics[width=.98\linewidth]{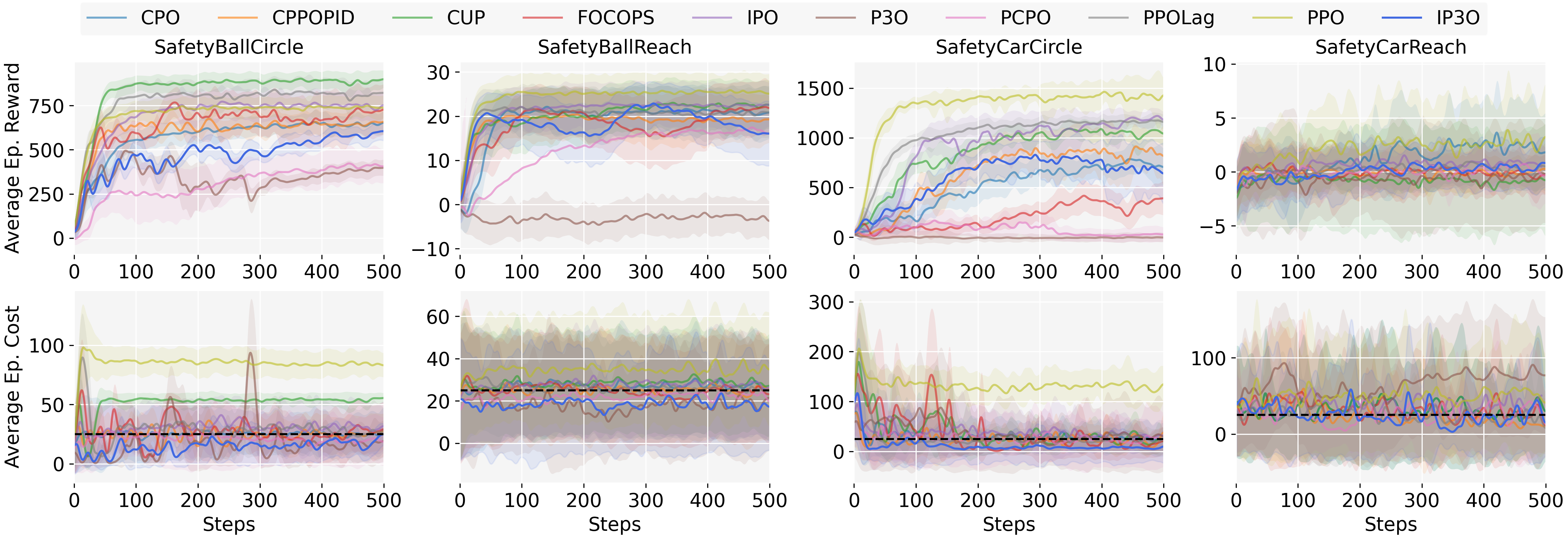}
    \caption{Comparison with the baselines using the Bullet Safety Gymnasium scenarios. Our method is marked as IP3O (blue). The dashed line indicates the cost constraint limit.}
    \label{fig_res3}
\end{figure*}

\subsection{MuJoCo Safety Velocity}

The Safety Velocity environments simulate autonomous robots using MuJoCo physics simulator. The objective is to control the agent to run as fast as possible while ensuring compliance with the velocity constraints of the robot. We evaluate our approach on the Ant, Half-Cheetah, Humanoid, and Swimmer robots. Each environment has varying episode lengths. For these experiments, we set the constraint limit to 25 and use $\alpha = 0.5$. Detailed environment descriptions are provided in the Supplementary Material.

As shown in Figure \ref{fig_res1}, IP3O obtains better returns while maintaining strict compliance with the velocity constraints compared to state-of-the-art algorithms. Algorithms such as FOCOPS and PCPO shows better returns compared to IP3O for the Ant and HalfCheetah environments, but struggle to find a feasible policy. Other penalty based algorithms, such as IPO and P3O demonstrate sub-optimal returns owing to sharp gradient change near the constraint boundary. This demonstrates that our approach balances safety and performance effectively in dynamic control tasks.

\subsection{Safety Gymnasium}

Safety Gymnasium environments present the challenge of navigating the agent towards goal state using RL policy in continuous state-space environments. The agent needs to avoid unsafe states such as pre-defined hazardous regions, or avoiding interactions with certain unsafe moving objects; while navigating. Generally the hazards and goals are detected using lidar signals returned by the environments at each step, that are part of the observation space for the agents. The rewards can be maximized by going towards the goal states, while going away from them incurs negative rewards. We evaluate our method in the Goal and Button tasks, where each episode runs for 1,000 time-steps. For these experiments, the constraint limit is 25, and we set $\alpha = 0.1$.

Figure \ref{fig_res2} shows that our algorithm consistently outperforms baseline methods in adhering to safety constraints within the policy space $\Pi_\mathcal{C}$ while achieving competitive or higher cumulative rewards. This highlights the robustness of our approach in navigation tasks with complex safety requirements. Algorithms such as PCPO and IPO obtains better returns but fail to find a safety compliant policy. Owing to lower value for $\alpha$, the learnt policy frequently shows constraint violations, but obtains much better rewards, especially in the Goal tasks.

\subsection{Bullet Safety Gymnasium}

The Bullet Safety Gymnasium environments consist of situations similar to Safety Gymnasium, and feature tasks such as circular movement near boundaries or gathering objects, using robots like Ball and Car, while staying within constraint barriers. Observation spaces include sensor data for nearby obstacles and task-specific information, such as distances to goals. These environments test safety-critical behaviors in constrained settings. For our evaluations, we set $\alpha = 1.0$ and a constraint limit of 25 across all scenarios. Detailed descriptions of the tasks are available in the Supplementary Material.

From the results in Figure \ref{fig_res3}, it is evident that IP3O achieves the best compliance with safety constraints among other methods; except in the CarReach environment, the policy demonstrates negligible constraint violation. However, this slightly reduces reward maximization, demonstrating that our algorithm prioritizes constraint satisfaction when necessary. Most of the algorithms struggle to find a consistently safe policy owing to oscillations in the loss function near the constraint boundary.

The constraint violations for all the above scenarios are summarized in Figure \ref{fig_summ1}. From the figure it can be seen that IP3O has the lowest cost violations overall, compared to the baselines for the MuJoCo velocity and Bullet Safety Gym environments; and marginally above the best performance in Safety Gym, since $\alpha=0.1$.

\begin{figure}[!ht]
    \centering
    \includegraphics[width=0.98\linewidth]{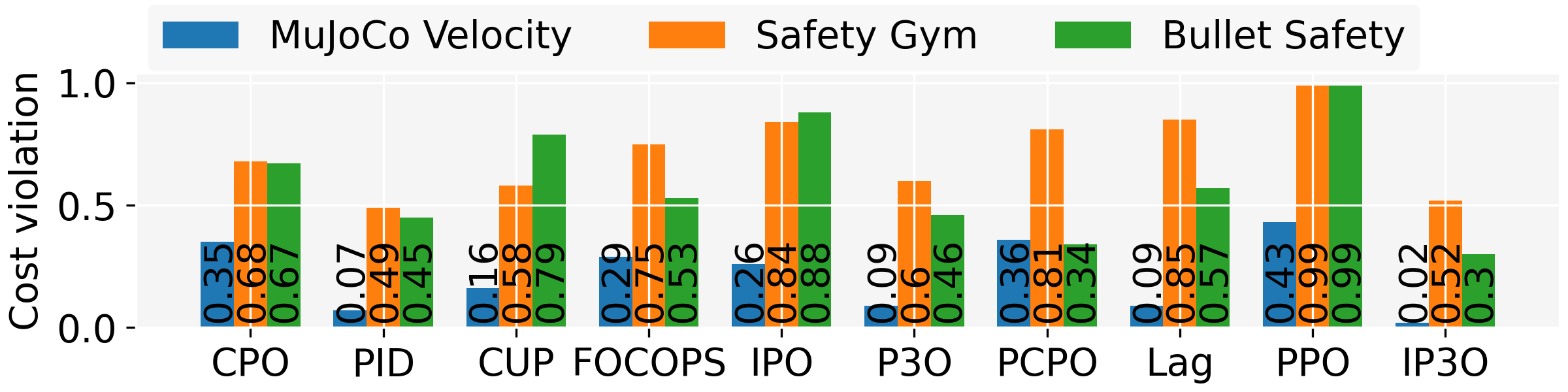}
    \caption{Summary of constraint violations across all environments.}
    \label{fig_summ1}
\end{figure}

\subsection{Multi-Agent Scenarios}

\begin{figure}
    \centering
    \includegraphics[width=0.98\linewidth]{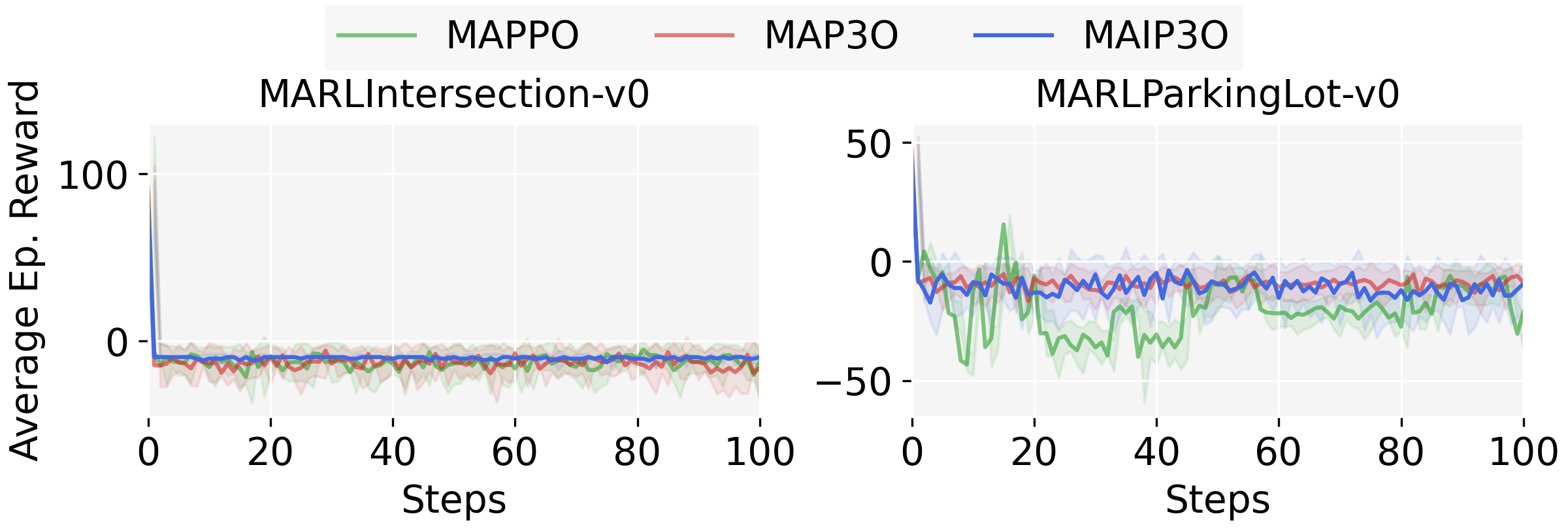}
    \caption{Evaluation results on the MetaDrive simulator}
    \label{fig_res4}
\end{figure}

The safety objective solved here is very relevant in the multi-agent automotive domain that consist of innate safety aspects. Our proposed approach is extendable to multi-agent scenarios. To demonstrate this, we evaluate it on cooperative driving tasks within the MetaDrive simulator \cite{li2022metadrive}, a lightweight and realistic platform designed for multi-agent decentralized reward settings. To train the policy using a cooperative feedback, the value error was estimated using the summation of the decentralized rewards. MetaDrive provides challenging environments for controlling multiple vehicles in predefined driving scenarios, where agents must cooperate to maximize rewards by avoiding safety violations such as crashes or going off-road. Additional environment details are provided in the Supplementary Material.


Training spanned 2000 episodes, using an on-policy replay buffer. We benchmarked our approach against multi-agent versions of PPO, namely MAPPO \cite{yu2022surprising} and MAP3O \cite{zhang2022penalized}. Each trained policy was evaluated over 10 episodes per evaluation step. As shown in Figure \ref{fig_res4}, our algorithm, marked as MAIP3O (Multi-Agent IP3O), achieves comparable performance to baseline algorithms in terms of collective reward across all agents (vehicles). We focus on reward maximization as the primary objective, with safety integrated to facilitate reward optimization.


In the following sub-section we discuss the effect of some of our hyper-parameters on the outcome with respect to return and constraint satisfaction.

\subsection{Ablation Studies}

We conduct ablation studies to evaluate the effect of hyper-parameters on the performance of our approach, focusing on two key parameters: the $\alpha$ hyper-parameter, and the cost limit $d$. These experiments are carried out across different environments to provide a comprehensive understanding of how variations in these hyper-parameters impact performance.

\begin{figure}[!ht]
    \centering
    \includegraphics[width=0.98\linewidth]{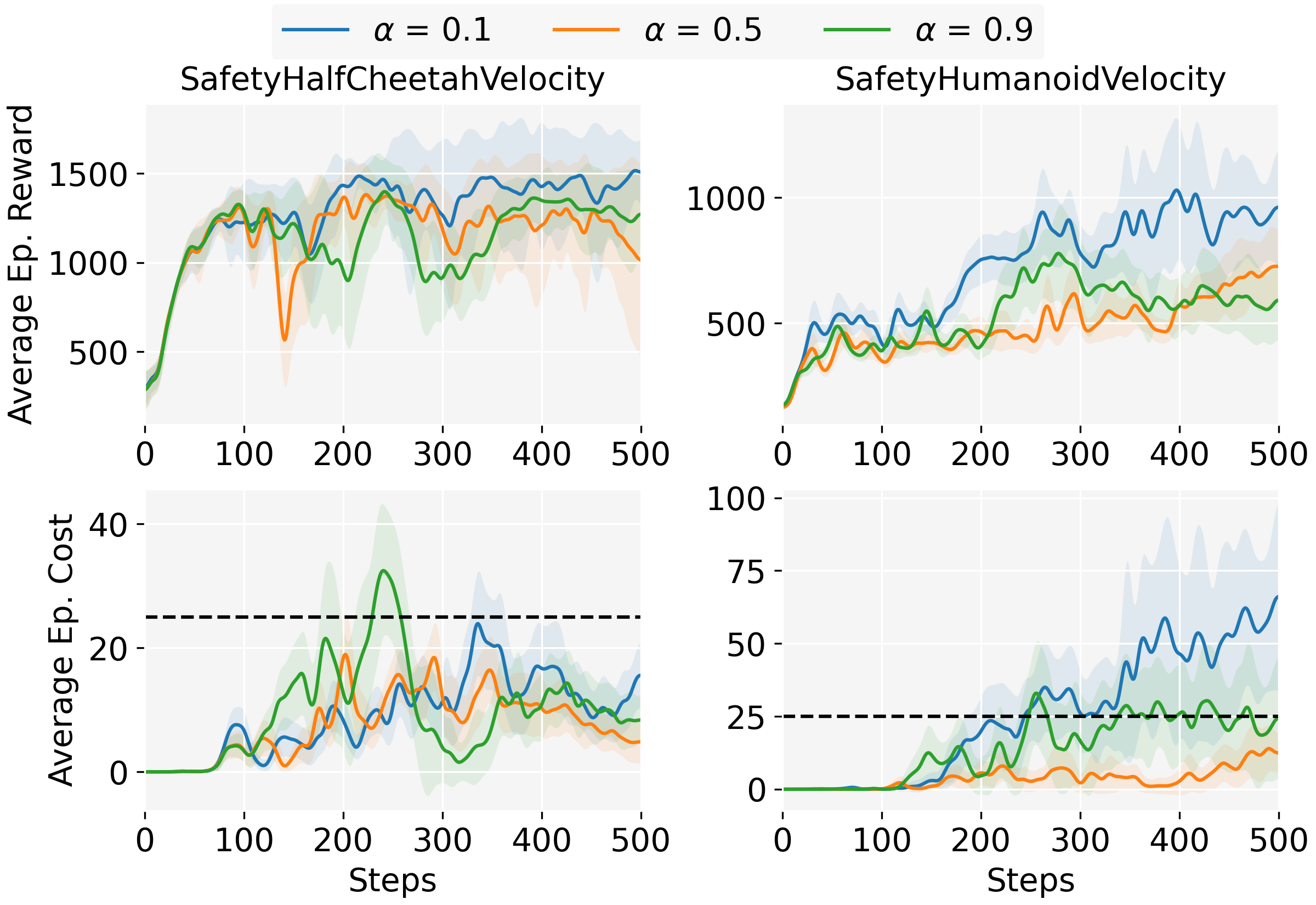}
    \caption{Ablation experiments using $\alpha$ on the Half-Cheetah and the Humanoid environments.}
    \label{fig_res5}
\end{figure}

The effect of $\alpha$ is studied in the Half-Cheetah and Humanoid tasks from the MuJoCo Safety Velocity benchmark. For these experiments, the cost limit $d$ is fixed at 25. As shown in Figure \ref{fig_res5}, setting $\alpha = 0.5$ achieves optimal constraint satisfaction while maintaining a reasonable reward. Reducing $\alpha$ to 0.1, although resulting in higher rewards, leads to frequent constraint violations, as shown in Figure \ref{fig_res5}. On the other hand, increasing $\alpha$ beyond 0.5 does not significantly improve constraint satisfaction in environments where reward and constraints are intrinsically tied to velocity. In scenarios with tight safety constraints, prioritizing constraint satisfaction with higher $\alpha$ values is beneficial. However, in less safety-critical tasks, relaxing the safety weight may allow for better reward maximization.

\begin{figure}[!ht]
    \centering
    \includegraphics[width=0.98\linewidth]{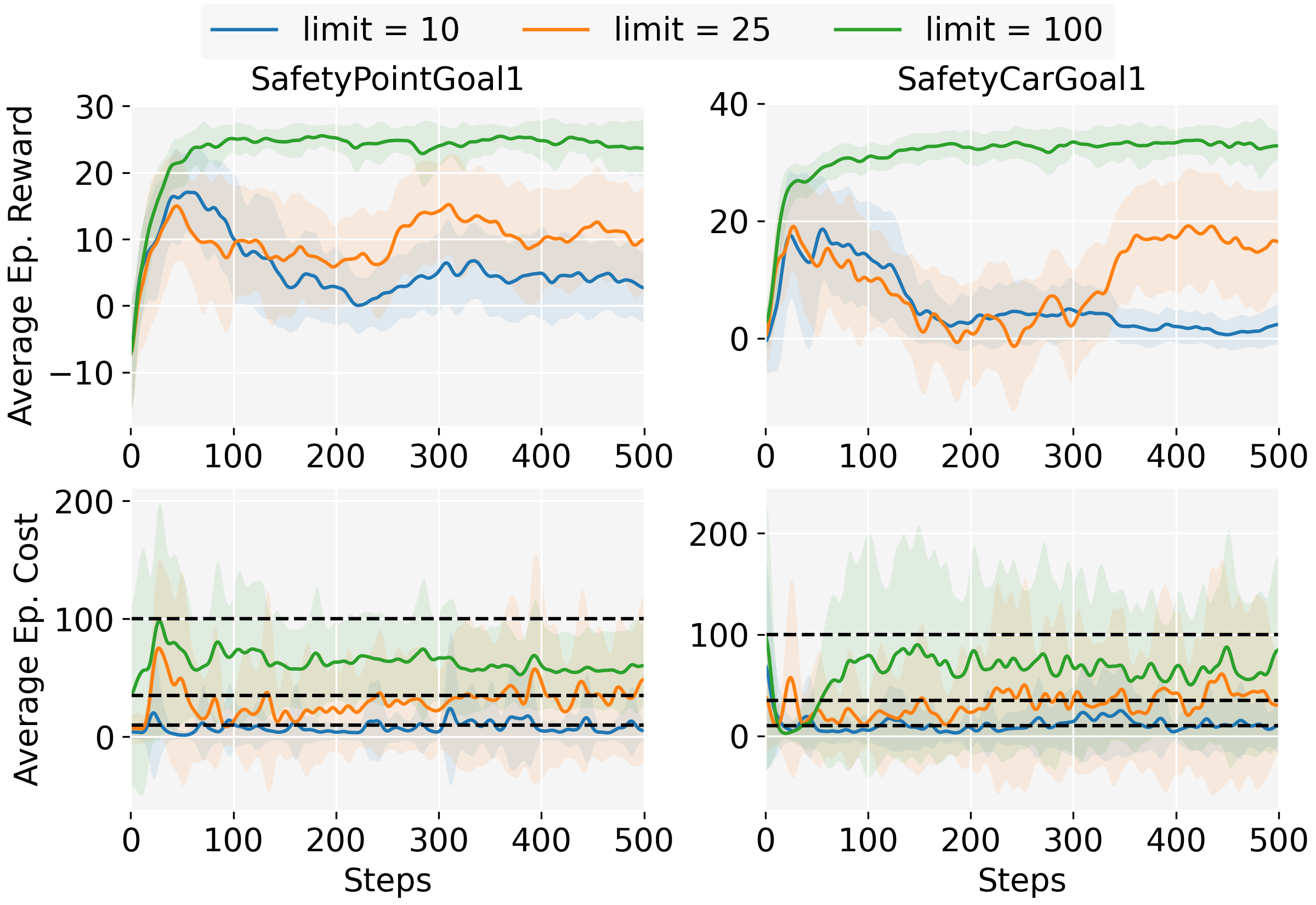}
    \caption{Ablation experiments using the cost limit $d$ hyper-parameter on the PointGoal1 and the CarGoal1 environments.}
    \label{fig_res6}
\end{figure}

In addition, we also analyze the robustness of IP3O across varying cost limit levels d in the PointGoal1 and CarGoal1 environments from the Safety Gymnasium benchmark, to better highlight the effect of cost limit on the policy learning. The results, presented in Figure \ref{fig_res6}, demonstrate that our approach effectively learns constraint satisfying policies, $\pi \in \Pi_{\mathcal{C}}$, for a range of $d$ values while maximizing rewards. This highlights the flexibility of our method in adapting to diverse safety requirements. Overall, these results underline the importance of careful tuning of hyper-parameters to achieve the desired balance between safety and reward.

\section{Conclusion}
\label{sec6}

In this paper, we introduced Incrementally Penalized Proximal Policy Optimization (IP3O), a novel policy optimization algorithm designed to solve safe RL tasks for combining constraint satisfaction with reward maximization. IP3O demonstrates adaptability to diverse safety requirements. Our experiments demonstrated the efficacy of IP3O across benchmark environments in safe RL, such as MuJoCo Safety Velocity, Safety Gymnasium, and Bullet Safety Gymnasium. The results highlight its superior ability to balance safety and reward compared to state-of-the-art approaches. Additionally, the algorithm’s scalability to multi-agent settings, as shown in MetaDrive scenarios, shows its potential for real-world applications like autonomous driving and robotics.

IP3O can be tailored to different safety-critical tasks as shown in the ablation studies that validate the impact of key hyper-parameters such as $\alpha$ and constraint limits. These findings demonstrate the algorithm’s practicality for safe RL tasks. In conclusion, IP3O provides a framework for advancing safe RL research through practical applications and theoretical guarantees. Future work could explore its scalability to larger multi-agent systems and further analyze its theoretical properties under varying safety thresholds.

\section*{Acknowledgements}

The authors acknowledge partial support from AI4CPS IIT Kharagpur grant no TRP3RDTR001 for this work.

\bibliographystyle{named}
\bibliography{ref}

\appendix

\section*{Supplementary Material}
In the following section we provide proofs of the theorem stated in the main text. We also provide some additional description regarding the experiments that was missing from the main text.

\section{Proofs of Theorems}

In this section we discuss the proofs of the theorems stated previously. We start with the dicusssion of the proof for Theorem 1.

\subsection{Proof of Theorem 1}

Before discussing the proof of the Theorem, we mention the performance difference lemma for expressing the performance bound over policy improvement.

\begin{lemma}[\cite{kakade2002approximately}]
    Given a reward function $\mathcal{R}$, for any two policies $\pi$ and $\pi'$ and any start state distribution $\rho$,
    \begin{equation}
        \label{eqn_supp1}
        \mathcal{J}_{\mathcal{R}}^{\pi'} - \mathcal{J}_{\mathcal{R}}^{\pi} = \frac{1}{1 - \gamma} \mathbb{E}_{\substack{s \sim d^{\pi'}\\a \sim \pi'}} \left[ A_{\mathcal{R}}^{\pi}(s, a) \right]
    \end{equation}
    \label{lemm_supp1}
\end{lemma}

In general, a similar performance difference equation can be established in term of the cost function/s $\mathcal{C}_i$ as: 
\begin{equation}
    \label{eqn_supp2}
    \mathcal{J}_{\mathcal{C}_i}^{\pi'} - \mathcal{J}_{\mathcal{C}_i}^{\pi} = \frac{1}{1 - \gamma} \mathbb{E}_{\substack{s \sim d^{\pi'}\\a \sim \pi'}} \left[ A_{\mathcal{C}_i}^{\pi}(s, a) \right]
\end{equation}

Using the above Lemma, the original optimization problem (given in Equation 1 of the main text) is reformulated as follows (Equation 4 of the main text, repeated here for ease of reference).
\begin{equation}
    \label{eqn_supp3}
    \begin{gathered}
    \pi_{k+1} = \arg \max_{\pi} \mathbb{E}_{\substack{s \sim d^{\pi}\\a \sim \pi}} \left[ r(\theta) A_{\mathcal{R}}^{\pi_k}(s, a) \right] \\
    \text{s.t.} \quad \mathcal{J}_{\mathcal{C}_i}(\pi_k) + \frac{1}{1 - \gamma} \mathbb{E}_{\substack{s \sim d^{\pi}\\a \sim \pi}} \left[ r(\theta) A_{\mathcal{C}_i}^{\pi_k}(s, a) \right] \leq d_i
    \end{gathered}
\end{equation}

Since the problem (4) is difficult to optimize, we reformulated the problem as a penalty problem as follows (Equation 6 of the main text, repeated here for ease of reference).
\begin{equation}
    \label{eqn_supp4}
    \pi_{k+1} = \arg \min_{\pi} \mathcal{L}_{\mathcal{R}}(\pi_k) + \eta \sum_{i=1}^m \text{CELU}( \mathcal{L}_{\mathcal{C}_i}(\pi_k))
\end{equation}
where $\mathcal{L}_{\mathcal{R}}(\pi_k) = \mathbb{E}_{\substack{s \sim d^{\pi}\\a \sim \pi}} [- r(\theta) A_{\mathcal{R}}^{\pi_k}(s, a)]$ and  $\mathcal{L}_{\mathcal{C}_i}(\pi_k) = \frac{1}{1 - \gamma} \mathbb{E}_{\substack{s \sim d^{\pi}\\a \sim \pi}} [ r(\theta) A_{\mathcal{C}_i}^{\pi_k}(s, a)] + (\mathcal{J}_{\mathcal{C}_i}(\pi_k) - d_i)$. For ease of notation we dilute the divergence requirement here. With respect to the equations \ref{eqn_supp3} and \ref{eqn_supp4}, we can state the following.

\begin{lemma}
    Let $\lambda^*$ be the Lagrange multiplier for the optimal solution of the dual problem of Equation \ref{eqn_supp3}. Given $\hat{\pi}$ is the solution of Equation \ref{eqn_supp3}. Then, as long as $\eta$ is sufficiently large such that $\eta \geq || \lambda^* ||_{\infty}$, $\hat{\pi}$ also solves Equation \ref{eqn_supp4}.
    \label{lemm_supp2}
\end{lemma}

\begin{proof}
    When $\mathcal{L}_{\mathcal{C}_i}^{\pi_k} \geq 0$;
    \begin{align*}
        \mathcal{L}_{\mathcal{R}}(\pi_k) & + \eta \sum_{i=1}^m \text{CELU}( \mathcal{L}_{\mathcal{C}_i}(\pi_k)) \\
        & \geq \mathcal{L}_{\mathcal{R}}(\pi_k) + \sum_{i=1}^m \lambda_i^* \mathcal{L}_{\mathcal{C}_i}(\pi_k) \quad [\text{since } \eta \geq || \lambda^* ||_{\infty}] \\
    \end{align*}
    When $\mathcal{L}_{\mathcal{C}_i}^{\pi_k} < 0$;
    \begin{align*}
        \mathcal{L}_{\mathcal{R}}(\pi_k) & + \eta \sum_{i=1}^m \text{CELU}( \mathcal{L}_{\mathcal{C}_i}(\pi_k)) \\
        & \geq \mathcal{L}_{\mathcal{R}}(\pi_k) + \sum_{i=1}^m \lambda_i^* \mathcal{L}_{\mathcal{C}_i}(\pi_k) \quad [\text{CELU}(x) \geq -\alpha]
    \end{align*}
    Since $\hat{\pi}$ is a solution of Equation \ref{eqn_supp3}, it satisfies the KKT conditions. Therefore, for any $\pi_k$ and $\mathcal{L}_{\mathcal{C}_i}^{\pi_k}$;
    \begin{align*}
        \mathcal{L}_{\mathcal{R}}(\pi_k) & + \sum_{i=1}^m \lambda_i^* \mathcal{L}_{\mathcal{C}_i}(\pi_k) \\
        & \geq \mathcal{L}_{\mathcal{R}}(\hat{\pi}) + \sum_{i=1}^m \lambda_i^* \mathcal{L}_{\mathcal{C}_i}(\hat{\pi}) \quad [\text{since }\hat{\pi} \text{ solves \ref{eqn_supp3}}] \\
        & = \mathcal{L}_{\mathcal{R}}(\hat{\pi}) \quad [\text{complementary slackness at }\hat{\pi}] \\
        & \geq \mathcal{L}_{\mathcal{R}}(\hat{\pi}) + \eta \sum_{i=1}^m \text{CELU}( \mathcal{L}_{\mathcal{C}_i}(\hat{\pi}))
    \end{align*}
    Thus the solution for Equation \ref{eqn_supp3}, $\hat{\pi}$, is also the solution for Equation \ref{eqn_supp4}. This ends the proof.
\end{proof}

\begin{lemma}
    Let $\lambda^*$ be the Lagrange multiplier for the optimal solution of the dual problem of Equation \ref{eqn_supp3}. Given $\bar{\pi}$ is the solution of Equation \ref{eqn_supp4}, and $\hat{\pi}$ solves Equation \ref{eqn_supp3}. Then, for $\eta \geq || \lambda^* ||_{\infty}$, $\bar{\pi}$ also solves Equation \ref{eqn_supp3}.
    \label{lemm_supp3}
\end{lemma}

\begin{proof}
    Given that $\bar{\pi}$ solves Equation \ref{eqn_supp4} and $\bar{\pi}$ is a feasible solution for Equation \ref{eqn_supp3}, i.e., $\mathcal{L}_{\mathcal{C}_i}(\bar{\pi}) \leq 0$;
    \begin{align*}
        \mathcal{L}_{\mathcal{R}}(\bar{\pi}) & + \eta \sum_{i=1}^m \text{CELU}( \mathcal{L}_{\mathcal{C}_i}(\bar{\pi})) \\
        & \leq \mathcal{L}_{\mathcal{R}}(\hat{\pi}) + \eta \sum_{i=1}^m \text{CELU}( \mathcal{L}_{\mathcal{C}_i}(\hat{\pi})) \\
        & = \mathcal{L}_{\mathcal{R}}(\hat{\pi}) \quad [\text{since }\hat{\pi} \text{ solves \ref{eqn_supp3}}]
    \end{align*}
    If $\mathcal{L}_{\mathcal{C}_i}(\bar{\pi}) > 0$;
    \begin{align*}
        \mathcal{L}_{\mathcal{R}}(\hat{\pi}) & + \eta \sum_{i=1}^m \text{CELU}( \mathcal{L}_{\mathcal{C}_i}(\hat{\pi})) \\
        & = \mathcal{L}_{\mathcal{R}}(\hat{\pi}) + \sum_{i=1}^m \lambda_i^* \mathcal{L}_{\mathcal{C}_i}(\hat{\pi}) [\text{complementary slackness}] \\
        & \leq \mathcal{L}_{\mathcal{R}}(\bar{\pi}) + \sum_{i=1}^m \lambda_i^* \mathcal{L}_{\mathcal{C}_i}(\bar{\pi}) \quad [\text{since }\hat{\pi} \text{ solves \ref{eqn_supp3}}] \\
        & \leq \mathcal{L}_{\mathcal{R}}(\bar{\pi}) + \eta \sum_{i=1}^m \text{CELU}( \mathcal{L}_{\mathcal{C}_i}(\bar{\pi})) [\text{since }\mathcal{L}_{\mathcal{C}_i}(\bar{\pi}) > 0]
    \end{align*}
    which is a contradiction to the assumption that $\bar{\pi}$ is the solution for Equation \ref{eqn_supp4}. Thus $\bar{\pi}$ is a feasible solution for Equation \ref{eqn_supp3}. This completes the proof.
\end{proof}

From Lemma \ref{lemm_supp2} and Lemma \ref{lemm_supp3}, we can deduce that Equation \ref{eqn_supp3} and Equation \ref{eqn_supp4} share the same optimal solution set. This completes the proof of Theorem 1.

\subsection{Proof of Theorem 2}

Before discussing the proof for Theorem 2 we discuss the following lemma for defining the limits on performance difference given that the trajectories for policy optimization are sampled from the current policy, $\pi_k$.

\begin{lemma}[\cite{achiam2017constrained}]
    For any reward function, $\mathcal{R}$, and policies $\pi$ and $\pi'$, let $\varepsilon_{\mathcal{R}}^{\pi'} = \max_s| \mathbb{E}_{a \sim \pi'}[A_{\mathcal{R}}^{\pi}(s, a)] |$, and $\delta = \mathbb{E}_{s \sim d^{\pi}}[ D_{KL}(\pi' || \pi)[s] ]$, and
    \begin{equation*}
        D_{\pi, \mathcal{R}}^{\pm}(\pi') = \frac{1}{1 - \gamma} \mathbb{E}_{\substack{s \sim d^{\pi}\\a \sim \pi}} \left[ \frac{\pi'(a|s)}{\pi(a|s)} A_{\mathcal{R}}^{\pi}(s, a) \pm \frac{ \sqrt{2 \delta} \gamma \varepsilon_{\mathcal{R}}^{\pi'}}{1 - \gamma} \right]
    \end{equation*}
    then the following holds:
    \begin{equation}
        D_{\pi, \mathcal{R}}^+(\pi') \geq \mathcal{J}_{\mathcal{R}}(\pi') - \mathcal{J}_{\mathcal{R}}(\pi) \geq D_{\pi, \mathcal{R}}^-(\pi')
        \label{eqn_supp5}
    \end{equation}
\end{lemma}

The above lemma can be similarly defined with respect to the cost function/s $\mathcal{C}$. Here we are trying to learn a conservative policy to stay safe considering environmental uncertainties by using CELU function to design our penalty with respect to the cost function. This applies an incentive with respect to the cost function for $\mathcal{L}_{\mathcal{C}_i}(\pi) < 0$. Although this method produces a safer policy, as seen from the empirical results, this penalty function however produces a limit on the optimal reward owing to its restrictiveness, since the policy keeps updating due to the cost function beyond $\mathcal{L}_{\mathcal{C}_i}(\pi) = 0$.

We had stated earlier in the Practical Implementation section that $\text{CELU}(\mathcal{L}_{\mathcal{C}_i}(\pi))$ worked good for our experiments, mathematically it induces a residual gradient. This can be removed by using $\max(\text{CELU}(\mathcal{L}_{\mathcal{C}_i}(\pi)), -\alpha(1 - h))$, where $0 < h < \alpha$ and $h \in \mathbb{R}$. This stops updates when $\text{CELU}(\mathcal{L}_{\mathcal{C}_i}(\pi)) \leq -\alpha(1 - h)$. From the equation of CELU we get the point at which the loss updates due to cost stop instead of at $\mathcal{L}_{\mathcal{C}_i}(\pi) = 0$.
\begin{gather*}
    \alpha \left( \exp \left( \frac{\mathcal{L}_{\mathcal{C}_i}(\pi)}{\alpha} \right) - 1 \right) = -\alpha(1 - h) \\
    \Rightarrow \mathcal{L}_{\mathcal{C}_i}(\pi) = \alpha \log(h)
\end{gather*}
This imposes an error bound of $|\alpha \log(h)|$ for each constraint on the optimal problem \ref{eqn_supp3}. From Lemma 2, given $\hat{\pi}$ is the optimal solution for Equation \ref{eqn_supp3}, combining the above lemma with the error bound we get the total error bound on the optimal solution.

\begin{align*}
    | \mathcal{L}(\hat{\pi}) & - \mathcal{L}(\pi_K) | \\
    & \leq | \mathcal{L}_{\mathcal{R}}(\hat{\pi}) -\mathcal{L}_{\mathcal{R}}(\pi_K)| \\
    & \quad + \eta \sum_{i=1}^m | \text{CELU}( \mathcal{L}_{\mathcal{C}_i}(\hat{\pi})) - \text{CELU}(\mathcal{L}_{\mathcal{C}_i}(\pi_K)) | \\
    & \leq | \mathcal{L}_{\mathcal{R}}(\hat{\pi}) -\mathcal{L}_{\mathcal{R}}(\pi_K)| + \eta \sum_{i=1}^m | \mathcal{L}_{\mathcal{C}_i}(\hat{\pi}) -\mathcal{L}_{\mathcal{C}_i}(\pi_K)|
\end{align*}
From Equation \ref{eqn_supp5}, we derive that,
\begin{align*}
    | \mathcal{L}(\hat{\pi}) & - \mathcal{L}(\pi_K) | \\
    & \leq \frac{ \sqrt{2 \delta} \gamma \varepsilon_{\mathcal{R}}^{\hat{\pi}}}{1 - \gamma} + \eta \sum_{i=1}^m \left[ \frac{ \sqrt{2 \delta} \gamma \varepsilon_{\mathcal{C}_i}^{\hat{\pi}}}{1 - \gamma} + | \alpha \log(h) | \right] \\
\end{align*}
This completes the proof of Theorem 2.

\section{Empirical Details}

We showed evaluation results across three widely-used safe RL environments: MuJoCo Safety Velocity \cite{ji2023safety}, Safety Gymnasium \cite{ray2019benchmarking}, and Bullet Safety Gymnasium \cite{gronauer2022bullet}. We also conducted experiments using the MetaDrive simulator \cite{li2022metadrive}. All the environments can be defined using a CMDP.

\subsection{Single Agent Environments}

Given an agent, the objective is to find a policy, $\pi \in \Pi_{\mathcal{C}}$ that maximizes reward. The reward and cost functions vary depending on the environments. We trained different agents in each environment to demonstrate the efficacy of our approach. The details of environments are described below.

\textit{Run Tasks}: The environments consist of simulation of autonomous robots based on MuJoCo simulator. The objective is to train agents to run on a plain surface, given the cost function defined on the velocity: $\sqrt{(v_{x}^2) + (v_{y}^2)}$. We trained four different agents, Ant, Half-Cheetah, Humanoid, and Swimmer from this environment. We also performed some ABlation experiments on these environments.

\textit{Safety Gymnasium}: This environment consists of multiple robots to be trained under various constrained environments. We used two agents, Point and Car for our experiments. Depending on the scenarios the cost function is designed. Such as in Goal tasks the objective is to reach a predefined goal, in Button tasks, the objective is to go near button objects. In all tasks, the agent has to avoid certain predefined regions that incur costs. The objects are identified via lidar beams, that are a part of the observation.

\textit{Bullet Safety Gymnasium}: Here, the task is to control objects such as Ball, and Car to perform tasks such as going in a circle very fast, while staying within predefined boundaries, or reaching a certain goal. We trained the above two agents on four scenarios from this benchmark environment.

\subsection{Multi-Agent Environments}

The Multi-Agent environments consist of multiple agents in a single scenario. We used MetaDrive scenarios, where the task is to control multiple cars, given a certain road situation. We performed experiments on scenarios that have an innate safety risk (due to crashing or going away from road) among agents, i.e. Parking Lot and unmanned Intersection. We used agent termination factor as our safety objective. The observation of each agent may only inform a partial scenario of the complete environment. The problem is generally solved by using an RNN in the policy network. We used GRU here. Also, for enforcing co-operation, we used summation of the local rewards for each agent, to train our policy following the theory of Value Decomposition Network \cite{sunehag2017value}.

Each vehicle (agent) has continuous action space. Crashes or road exits result in episode termination for the respective agent. During training, episodes are terminated when more than half the agents are eliminated, and during evaluation, termination occurs upon the first agent’s crash or completion of the scenario. We used lidar-based observations as inputs to the policy, where each agent acts upon their own local observation, and trained 10 agents in each environment. The evaluation was conducted in two challenging scenarios: Intersection and Parking Lot, both characterized by high probabilities of collisions, making them suitable for evaluating safety objectives.

\section{Hyper-parameter Settings}

Following are the hyper-parameters we used for our experiments. 

\begin{table}[!hb]
    \centering
\caption{Hyper-parameters used for our experiments}
\label{tab_supp1}
    \begin{tabular}{cc}
    \hline
         Hyper-parameter& Value\\
    \hline
         Policy Network Size& [64, 64]\\
 Value Network size& [64, 64]\\
 Policy Learning rate&0.0003\\
 Value Learning rate&0.0003\\
 batch size&64\\
 Activation&Tanh\\
 Epochs&500\\
 Steps / Epoch&5000\\
 Gamma ($\gamma$)&0.99\\
 clip factor ($\varepsilon$) &0.2\\
 $\eta$&20.0\\
 Advantage Estimation &GAE\\
 $\lambda$ & 0.95 \\
 $\lambda_{\mathcal{C}}$ & 0.95 \\
 \hline
    \end{tabular}
\end{table}

We conducted our experiments using a system equipped with 12th generation Intel(R) Core(TM) i7 CPU processor with 16GB system memory, and a NVIDIA GeForce RTX 3060 GPU with 6GB graphic memory.

\end{document}